\documentclass{article}

\usepackage[accepted]{icml2025}

\usepackage{amsmath}
\usepackage{amssymb}
\usepackage{amsthm}
\usepackage{mathtools}
\usepackage{booktabs}
\usepackage{graphicx}
\usepackage{hyperref}
\usepackage{microtype}
\usepackage{url}
\usepackage{xspace}

\newtheorem{theorem}{Theorem}[section]

\newtheorem{definition}[theorem]{Definition}
\newtheorem{remark}[theorem]{Remark}

\newcommand{\hitsplus}{\mathrm{hits}^{+}}
\newcommand{\hitsminus}{\mathrm{hits}^{-}}

\newcommand{\Mem}{\mathcal{M}}
\newcommand{\Ret}{\mathcal{M}_t}

\newcommand{\Ustar}{U^{*}}

\begin{document}

\icmltitle{When to Forget: A Memory Governance Primitive}

\icmlsetsymbol{equal}{*}

\begin{icmlauthorlist}
\icmlauthor{Baris Simsek}{}
\end{icmlauthorlist}

\icmlkeywords{agent memory, memory quality, confidence learning,
              empirical utility, lifelong learning}

\vskip 0.3in

\title{When to Forget: A Memory Governance Primitive}
\author{Baris Simsek}

\maketitle

\begin{abstract}
Agent memory systems accumulate experience but currently lack a
principled operational metric for memory quality governance---deciding
which memories to trust, suppress, or deprecate as the agent's task
distribution shifts. Write-time importance scores are static; dynamic
management systems use LLM judgment or structural heuristics rather
than outcome feedback. This paper proposes \emph{Memory Worth} (MW):
a two-counter per-memory signal that tracks how often a memory
co-occurs with successful versus failed outcomes, providing a
lightweight, theoretically grounded foundation for staleness detection,
retrieval suppression, and deprecation decisions. We prove that MW
converges almost surely to the conditional success probability $p^+(m) = \Pr[y_t =
+1 \mid m \in \mathcal{M}_t]$---the probability of task success given
that memory $m$ is retrieved---under a stationary retrieval regime with
a minimum exploration condition. Importantly, $p^+(m)$ is an
\emph{associational} quantity, not a causal one: it measures outcome
co-occurrence rather than causal contribution. We argue this is still a
useful operational signal for memory governance, and we validate it
empirically in a controlled synthetic environment where ground-truth
utility is known: after 10{,}000 episodes, the Spearman
rank-correlation between Memory Worth and true utilities reaches $\rho =
0.89 \pm 0.02$ across 20 independent seeds, compared to $\rho = 0.00$
for systems that never update their assessments. A retrieval-realistic micro-experiment with real text and standard
embedding retrieval (\texttt{all-MiniLM-L6-v2}) further shows stale
memories crossing the low-value threshold ($\mathrm{MW} = 0.17$) while
specialist memories remain high-value ($\mathrm{MW} = 0.77$) across
3{,}000 episodes. The estimator
requires only two scalar counters per memory unit and can be added to
architectures that already log retrievals and episode outcomes.
\end{abstract}

\section{Introduction}
\label{sec:intro}

When an agent stores a memory, it makes an implicit prediction: this
piece of knowledge will be useful in the future. That prediction may be
correct at the time of storage but become wrong as the world changes,
as the agent's task distribution shifts, or simply as new information
supersedes old. Many widely used agent memory systems rely primarily on
write-time heuristics or LLM-assigned importance scores to assess memory
quality~\citep{park2023generative,packer2023memgpt}. While more recent
systems have begun to incorporate dynamic memory management---including
admission control, deletion, and adaptive
retrieval~\citep{zhang2026amac,xu2025amem}---a simple, general
per-memory estimator of post-retrieval outcome association, with
convergence guarantees, has not been proposed.

The consequence is that outcome signals available at every episode
go unused. A memory that has been present in the retrieval set during
dozens of failures continues to be treated as trustworthy. A memory
that consistently co-occurs with successful outcomes receives no
additional credit. The agent's memory store accumulates experience, but
quality information from that experience is discarded.

This paper addresses the memory governance problem directly: how should
an agent decide which stored memories remain trustworthy as experience
accumulates? We argue that memory governance requires an operational
primitive: a per-memory online signal that accumulates retrieval
evidence and supports suppression, re-verification, prioritisation, and
deprecation decisions over time. We propose such a primitive---
\emph{Memory Worth}---a lightweight two-counter statistic that tracks,
for each memory, how often it co-occurs with successful versus failed
outcomes. MW supports staleness detection, retrieval suppression, and
deprecation without requiring causal attribution or architectural
changes. The contribution is both methodological and empirical: we
prove convergence to post-retrieval conditional success probability
under explicit assumptions, characterise three distinct failure modes
with quantified magnitudes, and demonstrate the governance primitive
under modern embedding retrieval. The central limitation is explicit:
MW measures outcome association, not causation, and therefore should be
treated as the minimal operational primitive from which practical
memory governance systems can be composed.

The estimator requires no architectural changes beyond outcome and
retrieval logging, which architectures that already log agent
interactions expose naturally.

\paragraph{Contributions.}
\begin{itemize}
\item \textbf{Governance primitive.} We define Memory Worth, a two-counter per-memory online signal that enables staleness detection, retrieval suppression, uncertainty-aware review, and deprecation decisions without causal attribution or architectural changes.

\item \textbf{Theoretical grounding.} We prove almost-sure convergence
  (Section~\ref{sec:theory}) under explicit assumptions (A1)--(A6) by
  a martingale argument, and show empirically that the Bayesian
  Beta-Bernoulli posterior mean converges to the same ranking at
  long-run evidence levels (Experiment~1, $\rho = 0.89 \pm 0.02$).

\item \textbf{Failure-mode science.} Three experiments characterise
  when and how MW breaks under realistic A3 violations
  (Experiments~2--4): task-difficulty confounding leaves global MW
  negatively correlated ($\rho \approx -0.33$) but hard-task-only
  conditioning recovers a positive signal ($\rho \approx +0.14 \pm 0.07$);
  retrieval-policy feedback does not cause collapse in the tested
  softmax regime; co-retrieval confounding requires $\approx$30\% retrieval
  diversity to resolve in the tested setting.

\item \textbf{Embedding-retrieval validation.} A retrieval-realistic
  micro-experiment (\texttt{all-MiniLM-L6-v2}, 3{,}000 episodes) shows
  stale memory MW reaching $0.17$ while specialist MW stabilises at
  $0.77$---and reproduces the hitchhiker co-retrieval pathology
  predicted by Experiment~4, demonstrating that semantic retrieval
  systems naturally induce the confound the theory identifies
  (Experiment~5, Section~\ref{sec:exp5}).
\end{itemize}

\section{Background and Related Work}
\label{sec:related}

\paragraph{Agent memory systems.}
\citet{park2023generative} introduce a three-component retrieval score
combining recency, LLM-assigned importance, and embedding relevance.
The importance score is the closest precursor to Memory Worth, but it
is static---assigned once at write time and never updated by outcomes.
\citet{packer2023memgpt} give agents explicit memory management
operations but provide no outcome-driven quality signal.
\citet{shinn2023reflexion} use failure outcomes to update memory at the
episode level---storing a self-reflection about what went wrong---rather
than tracking per-memory outcome statistics.
\citet{zhong2024memorybank} model memory decay via an Ebbinghaus
forgetting curve, capturing temporal dynamics but not outcome feedback.

More recent work has moved toward dynamic memory governance.
\citet{zhang2026amac} propose Adaptive Memory Admission Control
(A-MAC), which scores candidate memories along five dimensions
(utility, confidence, novelty, recency, type prior) before admission
to long-term storage. A-MAC operates \emph{before} storage, assigning
write-time scores that can be learned and adapted; Memory Worth operates
\emph{after} repeated retrieval, updating a per-memory statistic from
observed outcomes. The two are complementary: A-MAC controls what enters
memory; Memory Worth tracks how well what is in memory actually performs.
\citet{xu2025amem} propose A-MEM, which uses LLM prompting to
dynamically organize and evolve memory through Zettelkasten-style
linking. A-MEM focuses on structural organization and retrieval quality
rather than per-memory outcome tracking.

\paragraph{Credit assignment in reinforcement learning.}
The credit assignment problem~\citep{minsky1961steps,sutton2018rl}
asks which past actions caused a delayed reward. RL methods for credit
assignment~\citep{arjona2019rudder,harutyunyan2019hindsight} operate
over temporal sequences of actions, where credit flows backward through
time. In memory retrieval there is no temporal sequence over memories:
multiple memories are retrieved simultaneously. The classical machinery
of temporal difference learning and eligibility traces does not apply
directly.

\paragraph{Surveys and benchmarks.}
Recent surveys have mapped the rapidly expanding landscape of agent
memory. \citet{zhang2025memsurvey} provide a comprehensive taxonomy
distinguishing factual, experiential, and working memory, and survey
dynamics of memory formation, evolution, and retrieval; they identify
memory quality assessment as an open problem. \citet{hu2025evalmem}
introduce a benchmark covering four core memory competencies---accurate
retrieval, test-time learning, long-range understanding, and selective
forgetting---and find that existing systems fall short on all four.
Selective forgetting in particular is the competency most directly
addressed by Memory Worth: a convergent quality signal is a prerequisite
for principled deprecation decisions.

\paragraph{Honest gap statement.}
Prior dynamic memory management systems---including
A-MAC~\citep{zhang2026amac} and A-MEM~\citep{xu2025amem}---update
memory quality assessments, but do so through LLM judgment, structural
reorganization, or write-time heuristics rather than from per-memory
retrieval-outcome statistics. To our knowledge, prior work has not
isolated the post-retrieval per-memory success-rate estimator defined
here---tracking $p^+(m) = \Pr[y_t = +1 \mid m \in \mathcal{M}_t]$ via
two online counters per memory unit---and analyzed its convergence under
explicit assumptions. The novelty of Memory Worth is precise and operationally distinct:
a lightweight post-retrieval associational estimator with a convergence
guarantee and only two scalars of overhead.

\section{Memory Worth}
\label{sec:measure}

\subsection{Setting}
\label{sec:setting}

An agent interacts with tasks over discrete episodes $t = 1, 2, \ldots$
At each episode, it retrieves a set $\Ret \subset \Mem$ of $k$ memories
from the memory store $\Mem = \{m_1, \ldots, m_N\}$, takes an action,
and observes an outcome $y_t \in \{+1, -1\}$ where $+1$ denotes
success and $-1$ denotes failure. The agent also assigns a
\emph{retrieval weight} $w_t(m) \geq 0$ to each retrieved memory,
reflecting how influential that memory was in producing the action. We
require $\sum_{m \in \Ret} w_t(m) = 1$.

\subsection{Definition}

\begin{definition}[Memory Worth]
\label{def:mw}
The Memory Worth of memory $m$ after $T$ episodes is:
\begin{equation}
\label{eq:mw}
\mathrm{MW}_T(m) =
  \frac{\hitsplus_T(m)}{\hitsplus_T(m) + \hitsminus_T(m)}
\end{equation}
where the weighted retrieval counts accumulate as:
\begin{align}
\label{eq:counts}
\hitsplus_T(m)  &= \sum_{t=1}^{T} w_t(m)\,\mathbf{1}[m \in \Ret]\,\mathbf{1}[y_t = +1] \\
\hitsminus_T(m) &= \sum_{t=1}^{T} w_t(m)\,\mathbf{1}[m \in \Ret]\,\mathbf{1}[y_t = -1]
\end{align}
When $\hitsplus_T(m) + \hitsminus_T(m) = 0$, we set $\mathrm{MW}_T(m) = 0.5$
(uninformative prior).
\end{definition}

$\mathrm{MW}_T(m)$ is the empirical success rate of episodes in which
$m$ was retrieved, weighted by retrieval influence. It is bounded in
$[0,1]$, interpretable without domain knowledge, and requires only two
scalar counters per memory unit.

\subsection{Retrieval Weight}

The retrieval weight $w_t(m)$ can be set in several ways depending on
available information:

\begin{itemize}
\item \textbf{Uniform}: $w_t(m) = 1/k$ for all $m \in \Ret$. No
  additional information required.
\item \textbf{Score-proportional}: $w_t(m) \propto \mathrm{score}(m, q_t)$
  where $\mathrm{score}$ is the retrieval scoring function (e.g.,
  incorporating semantic similarity and recency). Gives higher weight
  to memories that most influenced the retrieval decision.
\item \textbf{Oracle}: $w_t(m) \propto \Ustar(m)$ where $\Ustar(m)$ is
  the true utility---only available in controlled experiments.
\end{itemize}

As shown in Section~\ref{sec:experiments}, all three weightings converge
to similar final values in the stationary long-run limit. The choice of
weight matters primarily for \emph{convergence speed} and for behavior
under distribution shift, which are addressed in future work.

\subsection{Two Counts Are Necessary}
\label{sec:taxonomy}

A single scalar $\mathrm{MW}_T(m)$ conceals important information that
the dual-count representation preserves. Consider two memories, both
with $\mathrm{MW} = 0.80$ computed from different count pairs:

\begin{center}
\begin{tabular}{lrrl}
\toprule
Memory & $\hitsplus$ & $\hitsminus$ & Status \\
\midrule
$m_A$ & 80.0 & 20.0 & High evidence; reliable \\
$m_B$ & 8.0  & 2.0  & Low evidence; uncertain \\
\bottomrule
\end{tabular}
\end{center}
$m_A$ and $m_B$ have identical Memory Worth but $m_A$ has ten times more
evidence. A system tracking only the ratio cannot distinguish them.

More importantly, the two counts together enable an evidence-aware
value taxonomy. A global ratio $r_m$ near $0.5$ does not by itself
identify a context-dependent memory---it identifies a memory with
\emph{mixed outcomes}. Whether those mixed outcomes arise from genuine
context-dependence (useful in task type A, harmful in task type B) or
from confounding (retrieved alongside genuinely useful memories
sometimes, unhelpful ones other times) cannot be determined from the
global ratio alone. Distinguishing these cases requires conditional
slices of $\mathrm{MW}$ by query cluster or task family, which is a
natural extension deferred to future work. The taxonomy below describes
what the dual counts \emph{can} directly support:

\begin{definition}[Value taxonomy]
\label{def:taxonomy}
Let $r_m = \hitsplus(m) / (\hitsplus(m) + \hitsminus(m))$ (the success-rate ratio, distinct from Spearman~$\rho$) and
$V_m = \hitsplus(m) + \hitsminus(m)$.
\begin{itemize}
\item \textbf{High-value}: $r_m > \theta_H$ and $V_m \geq V_{\min}$.
  The memory is consistently associated with successful outcomes.
  This may be useful as a signal to increase retrieval priority,
  subject to contextual controls or human review.
\item \textbf{Uncertain}: $V_m < V_{\min}$. Insufficient evidence to
  classify; use prior or global pool average.
\item \textbf{Mixed-outcome}: $\theta_L \leq r_m \leq \theta_H$
  and $V_m \geq V_{\min}$. Outcomes are neither consistently positive
  nor consistently negative. Conditional Memory Worth (by task or
  query cluster) is needed before acting.
\item \textbf{Low-value}: $r_m < \theta_L$ and $V_m \geq V_{\min}$.
  The memory is predominantly associated with failure. Retrieval
  suppression or deprecation may be appropriate, subject to contextual
  controls or human review, given the associational nature of MW.
\end{itemize}
\end{definition}

The key advantage over a single scalar is the separation of
\emph{uncertain} (insufficient evidence) from \emph{mixed-outcome}
(sufficient evidence, ambiguous signal): a raw ratio is identical for
both when $r_m = 0.5$, but $V_m$ distinguishes them. A system
acting on Memory Worth should not deprecate an uncertain memory---it
has simply not been retrieved enough to form an estimate.

\section{Convergence Theory}
\label{sec:theory}

\subsection{Main Result}

\begin{theorem}[Convergence of Memory Worth]
\label{thm:convergence}
Assume the following conditions hold for memory $m$:
\begin{enumerate}
\item[(A1)] \textbf{Stationarity.} The joint distribution of
  $(\mathbf{1}[m \in \Ret],\, y_t)$ is stationary across episodes.
\item[(A2)] \textbf{Exploration.} There exists $\delta > 0$ such that
  $\Pr[m \in \Ret \mid \mathcal{F}_{t-1}] \geq \delta$ for all $t$.
  This ensures $m$ is retrieved infinitely often even if retrieval
  adapts based on confidence.
\item[(A3)] \textbf{Conditional independence.} Given history
  $\mathcal{F}_{t-1}$, the retrieval indicator $\mathbf{1}[m \in
  \Ret]$ is independent of the outcome $y_t$.
\item[(A4)] \textbf{Bounded outcomes.} $y_t \in \{-1, +1\}$.
\item[(A5)] \textbf{Minimum weight.} There exists $w_{\min} > 0$ such
  that $w_t(m) \geq w_{\min}$ whenever $m \in \Ret$. This prevents
  score-proportional weighting schemes from assigning arbitrarily small
  weights that stall convergence. In practice, weights can be clipped
  from below by a small constant (e.g.\ $w_{\min} = 0.01$) without
  materially affecting the estimator.
\item[(A6)] \textbf{Outcome stationarity given retrieval.}
  $\Pr[y_t = +1 \mid m \in \Ret, \mathcal{F}_{t-1}] = p^+(m)$
  for all $t$. That is, the conditional success probability given
  retrieval is constant across episodes and equal to the marginal
  $p^+(m)$. This is the condition that makes $D_t$ a martingale
  difference; it rules out scenarios where the success rate fluctuates
  across episodes in a way correlated with history (e.g.\ alternating
  easy/hard tasks with different base rates).
\end{enumerate}
\noindent
Under (A1)--(A6), as $T \to \infty$:
\begin{equation}
\label{eq:convergence}
\mathrm{MW}_T(m) \;\xrightarrow{a.s.}\; p^+(m)
\end{equation}
where $p^+(m) = \Pr[y_t = +1 \mid m \in \Ret]$ is the conditional
success probability given memory $m$ is in the retrieval set.

\emph{Scope note.} Assumption (A3) holds when retrieval is uniform
random (as in our experiment) or governed by a policy that conditions
only on query semantics and is independent of the true outcome
conditional on history. It is violated when retrieval quality directly
predicts outcome conditional on history---a realistic situation in
agents that retrieve harder memories for harder tasks. In that regime, we conjecture that
$\mathrm{MW}_T(m)$ still converges, but to a quantity that conflates
retrieval-difficulty effects with $p^+(m)$; the precise limiting value
depends on the joint distribution of task difficulty and retrieval
probability, and we do not characterise it formally here.
We view (A3) as a calibration assumption rather than a physical law,
and flag it as the principal assumption requiring scrutiny in deployment.
\end{theorem}

\begin{proof}
Let $X_t = w_t(m)\,\mathbf{1}[m \in \Ret]\,\mathbf{1}[y_t = +1]$
and $Z_t = w_t(m)\,\mathbf{1}[m \in \Ret]$.
Define $S^+_T = \sum_{t=1}^T X_t$ and $S_T = \sum_{t=1}^T Z_t$.
Then $\mathrm{MW}_T(m) = S^+_T / S_T$.

\emph{Step 1 ($S_T \to \infty$ a.s.):} By (A2), $\Pr[m \in \Ret
\mid \mathcal{F}_{t-1}] \geq \delta > 0$ uniformly. By (A5),
$w_t(m) \geq w_{\min} > 0$ whenever $m$ is retrieved, so the
increments $Z_t = w_t(m)\mathbf{1}[m \in \Ret]$ satisfy
$\mathbb{E}[Z_t \mid \mathcal{F}_{t-1}] \geq \delta w_{\min} > 0$.
By the martingale strong law, $S_T \to \infty$ a.s.

\emph{Step 2 (Martingale difference):} Define $D_t = X_t -
Z_t\,p^+(m)$, where $X_t = w_t(m)\,\mathbf{1}[m \in \Ret]\,
\mathbf{1}[y_t=+1]$ and $Z_t = w_t(m)\,\mathbf{1}[m \in \Ret]$.
By (A6), $\Pr[y_t=+1 \mid m \in \Ret, \mathcal{F}_{t-1}] = p^+(m)$
for all $t$. Therefore:
\begin{align*}
\mathbb{E}[D_t \mid \mathcal{F}_{t-1}]
&= w_t(m)\,\Pr[m \in \Ret \mid \mathcal{F}_{t-1}]\\
&\quad\cdot\bigl(\Pr[y_t{=}+1 \mid m{\in}\Ret,\mathcal{F}_{t-1}] - p^+(m)\bigr)\\
&= 0.
\end{align*}
Hence $M_T = \sum_{t=1}^T D_t$ is a martingale difference sequence.
Note that (A3) is used to ensure the retrieval decision does not
itself carry information about $y_t$ beyond what $\mathcal{F}_{t-1}$
already contains; (A6) supplies the constant-$p^+(m)$ condition that
(A3) alone does not imply in non-i.i.d.\ settings.
For the bound on $|D_t|$: since weights are normalised
($\sum_{m'\in\Ret} w_t(m') = 1$), each individual weight satisfies
$w_t(m) \leq 1$. By (A4), $y_t \in \{-1,+1\}$, so $\mathbf{1}[y_t=+1]
\in [0,1]$ and $p^+(m) \in [0,1]$. Therefore
$|D_t| = w_t(m)\,\mathbf{1}[m\in\Ret]\,|\mathbf{1}[y_t=+1] - p^+(m)|
\leq w_t(m) \leq 1$ a.s.,
giving $\mathbb{E}[D_t^2 \mid \mathcal{F}_{t-1}] \leq 1$.
The upper bound $w_t(m) \leq 1$ follows directly from the
normalisation condition in Section~\ref{sec:setting}, not from
a separate assumption.

\emph{Step 3 (Convergence):} The predictable quadratic variation
satisfies $\langle M \rangle_T \leq S_T$. Since $S_T \to \infty$,
the conditions of the martingale strong law of large
numbers~\citep{hall1980martingale} (Theorem 2.18) are met:
$M_T / S_T \to 0$ a.s. Therefore:
\[
\mathrm{MW}_T(m) = \frac{S^+_T}{S_T} = p^+(m) + \frac{M_T}{S_T}
\;\xrightarrow{a.s.}\; p^+(m). \qquad \square
\]
\end{proof}

\begin{remark}[Rate intuition]
\label{rem:rate}
The estimation error $|\mathrm{MW}_T(m) - p^+(m)|$ decreases with
total weighted retrieval count $K_T(m) = S_T$. Intuitively: memories
retrieved more often have more evidence and more reliable estimates.
A formal finite-sample bound requires careful treatment of the random
denominator $S_T$ and the predictable quadratic variation
$\langle M \rangle_T$; we leave this to future work. We acknowledge that finite-sample
rates are important for calibrating the evidence threshold $V_{\min}$
in practice; the a.s.\ convergence guarantee of
Theorem~\ref{thm:convergence} establishes the signal is correct in
the limit, but does not bound how quickly $V_{\min}$ retrievals
suffice for reliable governance decisions.
\end{remark}

\begin{remark}[No causal knowledge required]
Theorem~\ref{thm:convergence} does not assume that the agent knows
\emph{which} retrieved memory caused the outcome. The convergence
holds purely from the co-occurrence of retrieval and outcomes. This
is the key property that makes Memory Worth computable in practice:
causal attribution is not needed.
\end{remark}

\subsection{What Memory Worth Converges To}

$p^+(m) = \Pr[y_t = +1 \mid m \in \Ret]$ is the probability of
success on episodes where $m$ is retrieved. This is an \emph{associational}
quantity, not a causal one. A memory that consistently co-occurs with
genuinely useful memories will accumulate positive counts even if it
does not itself contribute to success. Conversely, a memory retrieved
in difficult tasks---where success is unlikely regardless of memory
quality---will accumulate negative counts unfairly.

The claim of this paper is not that $p^+(m)$ identifies causal
contributors to success. The claim is narrower and still useful:
$p^+(m)$ is a convergent, computable operational signal. Prioritizing
memories with high $\mathrm{MW}$ over memories with low $\mathrm{MW}$
tends to shift the composition of retrieval sets toward memories
that historically co-occur with success. Whether this is because those
memories cause success or because they are structurally associated with
good retrieval contexts, the expected outcome of retrieval improves.
Causal identification of per-memory utility is an open problem that
Memory Worth does not attempt to solve.

\section{Experiments}
\label{sec:experiments}

\subsection{Design}

We construct a synthetic environment where the ground-truth utility
$\Ustar(m)$ of every memory is known, enabling direct measurement of
whether Memory Worth tracks true quality. This setup is deliberately
controlled: uniform random retrieval ensures that assumption (A3) holds
exactly, so the experiment tests whether the estimator converges as
proved---not whether it is robust to violations of (A3) in real agents.
The latter is addressed empirically in Section~\ref{sec:a3violations}.

\paragraph{Memory pool.} $N = 100$ memories. Each memory $m_i$ is
assigned a true utility $\Ustar(m_i) \sim \mathrm{Uniform}(0,1)$,
drawn once per seed and held fixed throughout all episodes.

\paragraph{Retrieval.} At each episode, $k = 8$ memories are sampled
uniformly at random from the pool, satisfying assumption (A2) with
$\delta = k/N = 0.08$ and ensuring (A3) holds exactly. This
intentionally sanitizes the retrieval process to isolate the
estimator's behavior from confounding by retrieval policy.

\paragraph{Outcomes.} Episode success is drawn from:
\[
y_t = +1 \;\text{ with probability }\;
\mathrm{clip}\!\left(
  \tfrac{1}{k}\textstyle\sum_{m \in \Ret}\Ustar(m) + \xi_t,\;
  0,\; 1\right)
\]
where $\xi_t \sim \mathcal{N}(0, \sigma^2)$ with $\sigma = 0.10$.
By design, success probability equals mean utility of the retrieved set
plus noise. This means co-occurrence with success is informative about
utility for every memory, which is the best-case scenario for the
estimator. Real-agent settings---where task difficulty confounds the
retrieval-outcome relationship---are harder and are not tested here.

\paragraph{Weighting strategies compared.} All four strategies receive
identical worlds (same seed, same retrievals, same outcomes):
\begin{itemize}
\item \textbf{No update}: $\mathrm{MW}_T(m) = 0.5$ permanently.
  Represents any system that never uses retrieval outcomes to revise
  memory quality assessments.
\item \textbf{Uniform}: $w_t(m) = 1/k$ for all retrieved memories.
\item \textbf{Score-proportional} (Ours): $w_t(m) \propto \mathrm{sim}(m)$,
  where $\mathrm{sim}(m)$ is a semantic similarity proxy correlated
  with $\Ustar(m)$ at $r = 0.65$.
\item \textbf{Oracle}: $w_t(m) \propto \Ustar(m)$. Ground-truth weights;
  establishes an upper bound on weighting accuracy.
\end{itemize}

The no-update baseline is intentionally simple---it shows the cost of
discarding outcome information entirely. Stronger baselines, including
a Bayesian Beta-Bernoulli estimator with uncertainty-aware ranking and
a contextual estimator conditioning on query cluster, are important
comparisons deferred to future work alongside real-agent experiments.

\paragraph{Metric.} Spearman rank-correlation between $\mathrm{MW}_T(m)$
and $\Ustar(m)$ across all $N = 100$ memories, at each checkpoint
(every 500 episodes), averaged over 20 independent seeds.

\subsection{Results}

\begin{figure}[t]
\centering
\includegraphics[width=\columnwidth]{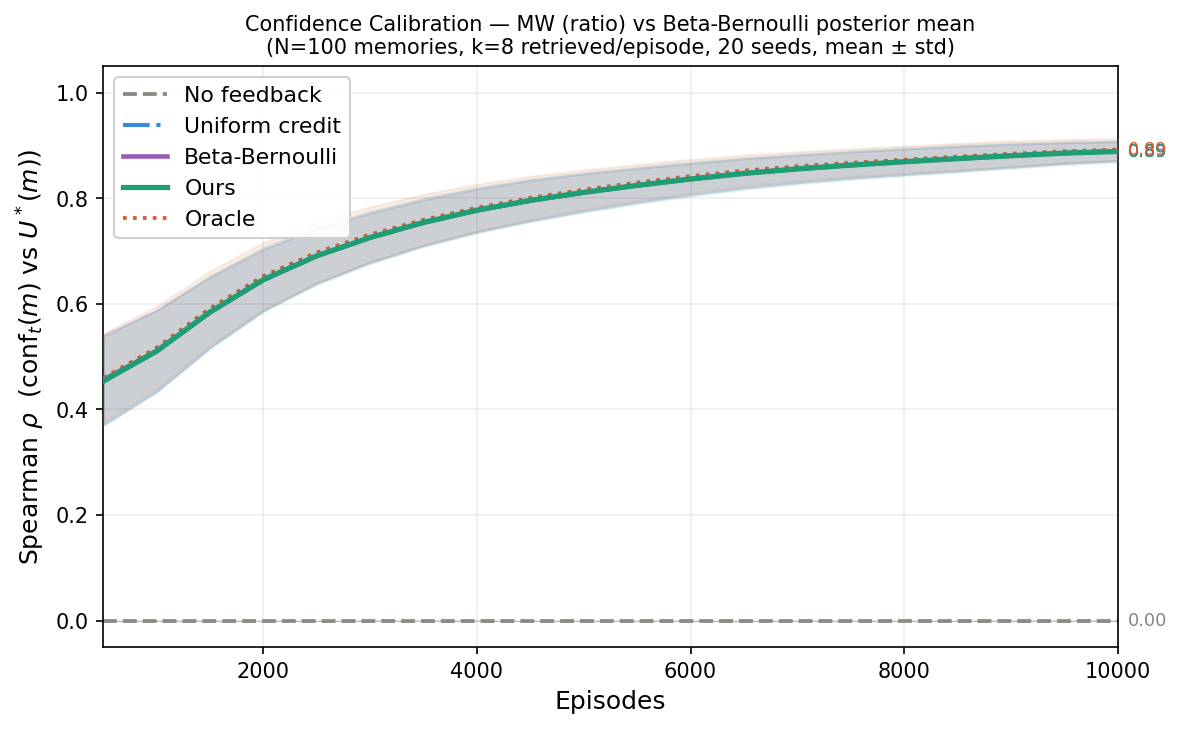}
\caption{Memory Worth calibration over episodes. Spearman $\rho$ between
$\mathrm{MW}_T(m)$ and $\Ustar(m)$ for all four weighting strategies,
averaged over 20 seeds (shaded regions: $\pm$1 std). The no-feedback
baseline stays at $\rho = 0$ throughout. All three updating strategies
converge to $\rho \approx 0.89$ by episode 10{,}000.}
\label{fig:calibration}
\end{figure}

\begin{table}[t]
\caption{Memory Worth calibration summary. Values are Spearman $\rho$
(mean $\pm$ std over 20 seeds). $\Delta$ is the difference from Uniform
at episode 10{,}000. Beta-Bernoulli converges to the same ranking as
Uniform at long-run evidence levels and is discussed in the text.}
\label{tab:results}
\vskip 0.1in
\begin{center}
\begin{small}
\begin{tabular}{lcccc}
\toprule
Method & $\rho$@2k & $\rho$@5k & $\rho$@10k & $\Delta$ \\
\midrule
No update        & $0.00_{\pm.00}$ & $0.00_{\pm.00}$ & $0.00_{\pm.00}$ & $-0.89$ \\
Uniform          & $0.66_{\pm.06}$ & $0.81_{\pm.03}$ & $0.89_{\pm.02}$ & $0.00$ \\
Ours (sim-wt.)   & $0.66_{\pm.06}$ & $0.81_{\pm.04}$ & $0.89_{\pm.02}$ & $0.00$ \\
Oracle           & $0.67_{\pm.06}$ & $0.82_{\pm.04}$ & $0.90_{\pm.02}$ & $+0.00$ \\
\bottomrule
\end{tabular}
\end{small}
\end{center}
\vskip -0.1in
\end{table}

Results are shown in Figure~\ref{fig:calibration} and
Table~\ref{tab:results}.

\paragraph{Memory Worth converges strongly (RQ1 confirmed).} All three
updating strategies rise from $\rho \approx 0.66$ at episode 2{,}000
to $\rho \approx 0.89$--$0.90$ at episode 10{,}000, with standard
deviations of $\pm 0.02$. The no-feedback baseline remains flat at
$\rho = 0.00$ throughout---it has no mechanism for revising its initial
assessment and therefore never learns anything about memory quality.
The contrast is stark: a difference of $0.89$ in Spearman~$\rho$
separating a system that tracks outcomes from one that does not.

\paragraph{The long-run limit is the same for all weighting strategies.}
Uniform, score-proportional, and oracle weighting converge to
statistically indistinguishable final values ($\rho = 0.89$, $0.89$,
$0.90$ respectively; differences within one standard deviation). This
is expected from Theorem~\ref{thm:convergence}: in the stationary
long-run limit, all three weightings estimate the same quantity
$p^+(m)$, just with different variance in intermediate estimates.
The choice of weighting matters for \emph{convergence speed} and for
\emph{non-stationary environments}; in the stationary setting examined
here, 10{,}000 episodes is sufficient for all strategies to converge.

\paragraph{Beta-Bernoulli posterior mean matches MW at convergence.}
The Beta-Bernoulli estimator with uniform prior ($\alpha = \beta = 1$)
yields $\rho = 0.89 \pm 0.02$ at episode 10{,}000---identical to raw
MW with uniform credit ($\Delta = 0.000$, Table~\ref{tab:results}).
This is the expected theoretical result: the posterior mean
$(\alpha + \mathrm{hits}^+)/(\alpha + \beta + K_T(m))$ converges
to the same limit as the raw ratio as $K_T(m) \to \infty$, with a
shrinkage factor of $(\alpha+\beta)/K_T(m)$ that vanishes with
evidence. The practical advantage of the Bayesian formulation is at
\emph{low evidence counts}: with fewer than $\sim$10 retrievals, the
posterior mean is pulled toward $0.5$ rather than producing extreme
estimates from small sample ratios. In this stationary long-run
experiment the two estimators are indistinguishable; the difference
would be visible in early episodes or in low-retrieval-frequency
regimes. We intentionally compare posterior mean rather than credible-bound
ranking so that the asymptotic estimator target matches MW exactly,
making the comparison clean. The genuinely uncertainty-aware Bayesian
advantages---lower credible bound ranking, posterior variance, Thompson
sampling, and risk-sensitive deprecation---have different asymptotic
targets and are deferred to future work as noted in
Section~\ref{sec:implications}.

\paragraph{Dual counts add information beyond the ratio alone.}
Using thresholds $\theta_H = 0.60$, $\theta_L = 0.40$, and $V_{\min}
= 10$ retrievals, the updating strategies classify on average 1.0--6.2
memories as low-value per seed. The score-proportional weighting
classifies 5.2 low-value memories vs.\ 2.1 for uniform weighting.
This difference reflects the weighting mechanism: sim-proportional
weights concentrate counts on high-similarity memories, causing
low-similarity memories to accumulate relatively more $\hitsminus$
weight and fall below $\theta_L$. The evidence-volume dimension $V_m$
correctly separates low-value from uncertain---memories with fewer than
10 retrievals are withheld from classification regardless of their
ratio. Whether low-value classification in this experiment reflects
genuine quality differences or retrieval confounding cannot be
determined from this design alone, as argued in
Section~\ref{sec:taxonomy}.

\subsection{Assumption (A3) Violation Studies}
\label{sec:a3violations}

Experiment~1 validates Memory Worth in the regime where all assumptions
hold by design. The three experiments below test three realistic ways
that assumption (A3) is violated, characterizing how and when the
estimator breaks, what the failure magnitude is, and how much partial
remedies recover. All three use $N=100$, $k=8$, $T=10{,}000$, 20 seeds.

\subsubsection{Experiment~2: Task-Difficulty Confound}

\begin{figure}[t]
\centering
\includegraphics[width=\columnwidth]{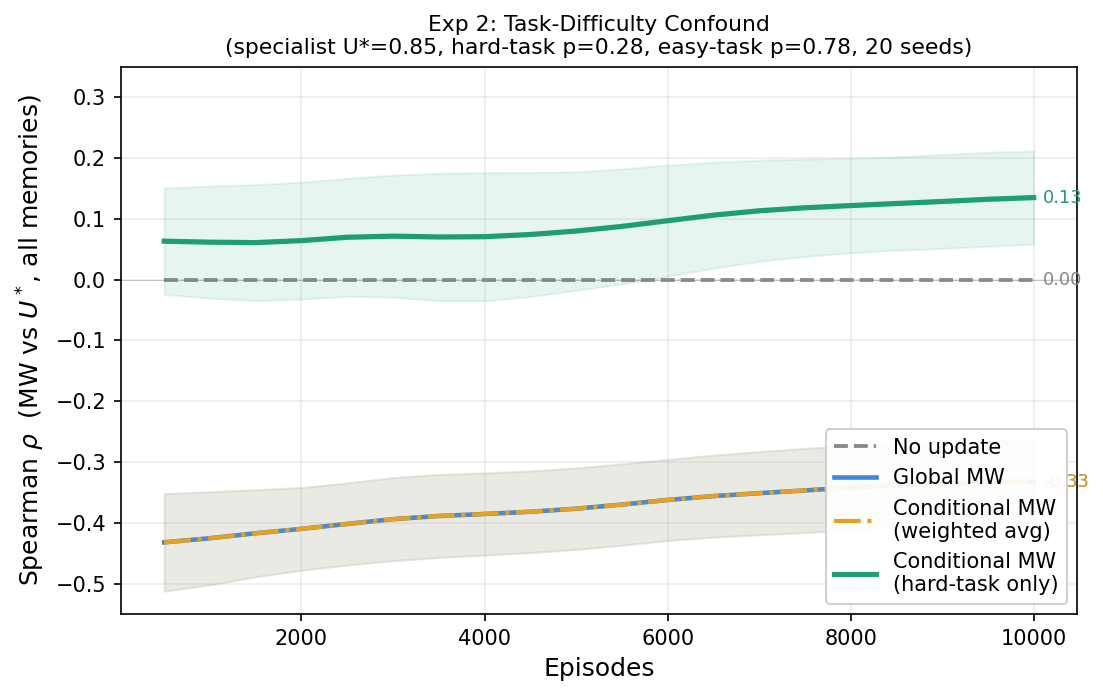}
\caption{Experiment~2: Task-difficulty confound (20 seeds, mean $\pm$ std).
Global MW (orange) remains negatively correlated with true utility
($\rho \approx -0.33$) because specialist memories ($\Ustar = 0.85$)
appear only on low-success hard tasks and are penalised by task difficulty.
A weighted-average conditional MW (yellow) fails equally: mixing easy- and
hard-task counts preserves the same base-rate confound. When MW is
conditioned on the hard-task population only and evaluated over memories
with hard-task evidence (green), the ranking signal becomes positive
($\rho \approx +0.14$), showing partial recovery once task-type
confounding is removed. The gap to the unconfounded baseline ($0.89$)
remains large, motivating within-stratum normalisation as future work.}
\label{fig:exp2}
\end{figure}

\paragraph{Setup.} The memory pool contains 70 generalist memories with
$\Ustar \sim \mathrm{Uniform}(0,1)$ and 30 specialist memories with
$\Ustar = 0.85$. Tasks are either easy (base success probability 0.78,
retrieval from generalists only) or hard (base success probability 0.28,
retrieval mixes generalists and specialists). Specialists therefore
appear exclusively on hard tasks.

\paragraph{Result.} Figure~\ref{fig:exp2} shows three curves. Global MW
(orange) remains negatively correlated with true utility
($\rho \approx -0.33$): specialist memories appear only on hard tasks
(success rate $0.28$) and are penalised by task difficulty rather than
low true utility, so they rank below generalists with much lower
$\Ustar$. A weighted-average conditional MW (yellow) fails equally at
$\rho \approx -0.33$: combining easy- and hard-task counts for every
memory leaves the hard-task base-rate confound mathematically intact. When MW
is conditioned on the hard-task population only---tracking
$\mathrm{MW}(m \mid \text{task} = \text{hard})$ and evaluating
only over memories with hard-task evidence---the ranking signal becomes
positive ($\rho \approx +0.14 \pm 0.07$), showing partial recovery
once task-type confounding is removed. The gap to the unconfounded
baseline of $0.89$ remains large because the hard-task base rate
($0.28$) still dilutes specialist scores relative to generalists that
appear on both task types.

\paragraph{Implication.} Removing the task-type confound by conditioning
on the hard-task population moves the ranking from $\rho \approx -0.33$
to $\rho \approx +0.14 \pm 0.07$, a shift of $\sim$0.47 points in the
right direction. Full recovery to the unconfounded baseline ($0.89$)
requires additionally controlling for the hard-task base rate within
each stratum---a within-stratum normalisation deferred to future work.
The practical message is that na\"ive global MW is not just noisy but
\emph{directionally wrong} under task-difficulty confounding; the correct
conditioning variable (task type) must be identified before MW estimates
can be used for governance decisions in this regime.

\subsubsection{Experiment~3: Retrieval Policy Feedback Loop}

\begin{figure}[t]
\centering
\includegraphics[width=\columnwidth]{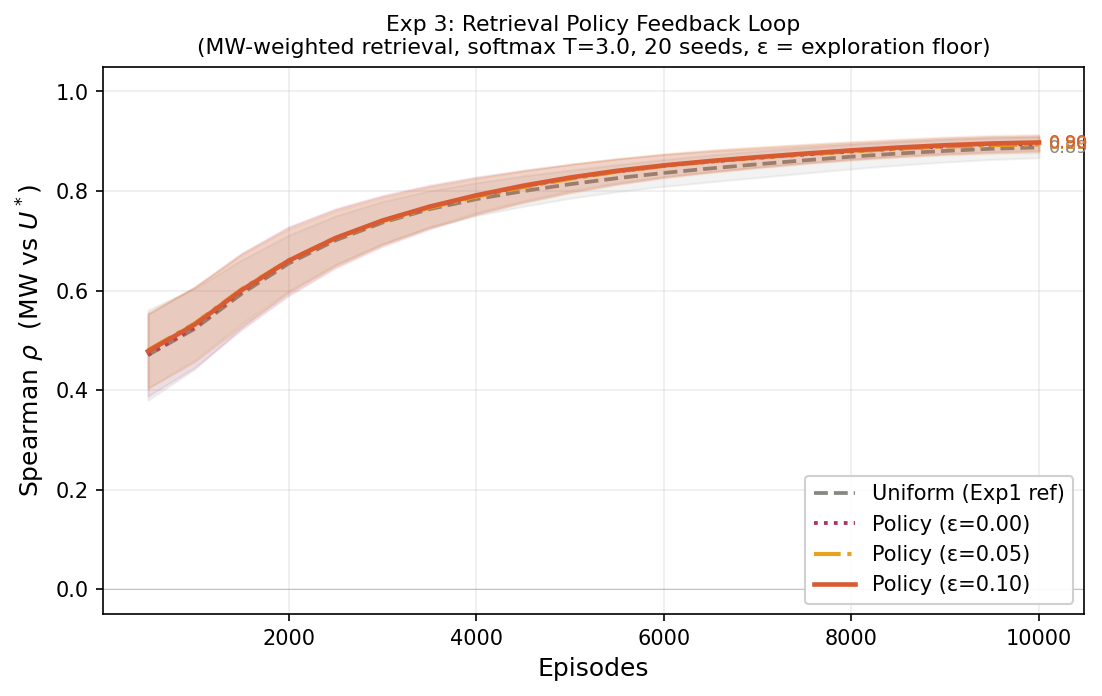}
\caption{Experiment~3: Retrieval policy feedback loop. MW-weighted
retrieval converges to $\rho \approx 0.896$--$0.899$ across all
exploration floors, matching or marginally exceeding the uniform reference
($\rho = 0.890$, 20 seeds, mean $\pm$ std). Even with $\varepsilon = 0$
(no forced exploration floor), the estimator does not degenerate. The
self-correcting feedback loop outweighs the rich-get-richer effect in
this regime.}
\label{fig:exp3}
\end{figure}

\paragraph{Setup.} Retrieval is replaced by a softmax policy biased by
current MW scores (temperature $\tau = 3.0$), mixed with a uniform
exploration floor $\varepsilon \in \{0.00, 0.05, 0.10\}$. This creates
the feedback loop that Theorem~\ref{thm:convergence} explicitly excludes.

\paragraph{Result.} Figure~\ref{fig:exp3} shows that all three policy
variants reach $\rho \approx 0.895$--$0.899$ by episode 10{,}000,
compared to $0.890$ for the uniform reference. Even without any
exploration floor ($\varepsilon = 0$), the estimator does not spiral
into degenerate concentration. The feedback loop is self-correcting:
high-MW memories retrieved more frequently drop in MW when they produce
failures, reducing their retrieval priority. This dampens rich-get-richer
dynamics and keeps the estimator calibrated.

\paragraph{Implication.} In the tested softmax-feedback regime---stationary
world, single temperature, uniform memory pool, 10{,}000 episodes---MW-based
retrieval did not collapse and remained self-correcting. This is a positive
result within that regime. Whether the same holds under non-stationary
utility, sparse specialist memories, adversarial skew, or harder task
distributions remains untested. The theoretical exploration condition (A2)
is sufficient for the convergence guarantee but may not be necessary in
practice; characterising the exact conditions for robustness is future work.

\subsubsection{Experiment~4: Co-Retrieval Confound}

\begin{figure}[t]
\centering
\includegraphics[width=\columnwidth]{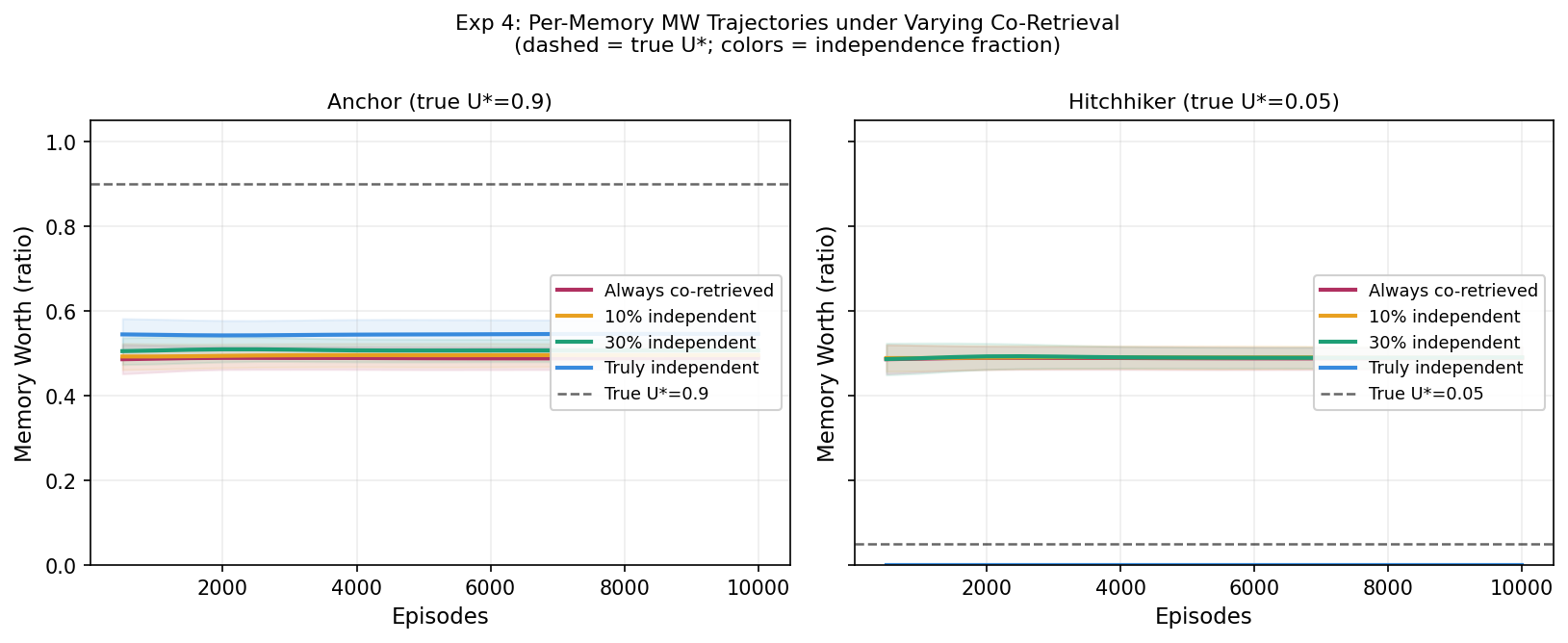}
\caption{Experiment~4: Per-memory MW trajectories for anchor
($\Ustar = 0.90$) and hitchhiker ($\Ustar = 0.05$) under varying
independence fractions. Dashed lines mark true $\Ustar$. With 0\%
independent retrievals, both memories converge to identical MW~$\approx
0.49$, indistinguishable despite a 17$\times$ difference in true
utility. Separation requires at least 30\% independent episodes, and
full convergence to true utility only with 100\% independence.}
\label{fig:exp4}
\end{figure}

\paragraph{Setup.} One ``anchor'' memory ($\Ustar = 0.90$) and one
``hitchhiker'' ($\Ustar = 0.05$) are always retrieved together except
in a controlled fraction of episodes (the \emph{independence fraction})
where only the anchor is retrieved. The hitchhiker contributes nothing
causally but accumulates positive counts whenever co-retrieved with the
anchor during successful episodes.

\paragraph{Result.} Figure~\ref{fig:exp4} shows the per-memory MW
trajectories. With 0\% independent retrievals, both memories converge
to identical $\mathrm{MW} \approx 0.49$---MW cannot distinguish them
at all. With 10\% independence, the separation is negligible ($0.496$
vs $0.490$). Only at 30\% does meaningful divergence begin, and with 100\% independence the anchor reaches MW~$= 0.546$ and the
hitchhiker MW~$= 0.000$---correctly ordered, though the anchor has not
yet converged to its true $\Ustar = 0.90$ within 10{,}000 episodes.
Full convergence to $\Ustar$ requires more episodes under any weighting.

\paragraph{Implication.} Co-retrieval confounding is the most severe
failure mode studied here. In this controlled setting---one utility gap,
one episode budget, uniform weighting---meaningful separation emerged only
once $\approx$30\% of episodes broke the habitual co-retrieval pair.
The 30\% figure is specific to this simulator and should not be read as
a universal threshold; the required diversity fraction will depend on the
utility gap between anchor and hitchhiker, episode count, and the
retrieval topology. The broader design implication holds regardless:
retrieval diversity---ensuring memories are occasionally retrieved without
their habitual co-retrievals---is a necessary condition for MW to
distinguish confounded from genuinely high-quality memories.

\subsection{Experiment~5: Text-Based Retrieval Agent}
\label{sec:exp5}

The four preceding experiments use abstract utility numbers and synthetic
retrieval. Experiment~5 replaces these with real text memories and neural embedding
retrieval (\texttt{all-MiniLM-L6-v2}~\citep{reimers2019sbert}, cosine
similarity), with a keyword-match outcome function requiring no LLM API.
The goal is to test whether the MW signal survives contact with real text
and modern semantic retrieval, using 3{,}000 episodes across 20 seeds.

\paragraph{Setup.}
The memory store contains 20 text sentences covering geography, Python
programming, and general science. Four memories are designated by type:
\emph{stale} (pre-1993 Czechoslovakia facts, correct in Phase~1 and wrong
after the phase shift), \emph{specialist} (Python list reversal, useful only
for Python tasks), \emph{hitchhiker} (general Python list methods,
semantically near the specialist but unhelpful for reversal tasks), and
\emph{control} (recursion, consistently useful across tasks).

The agent runs 3{,}000 episodes. At each episode a task is sampled from a
phase-dependent distribution: Phase~1 (episodes 1--100) draws 60\% geography
questions for which the stale memory is correct; Phase~2 (episodes 101--3{,}000)
shifts to 30\% dissolution-framing geography (stale wrong), 35\% Python, and
35\% general. Retrieval uses \texttt{all-MiniLM-L6-v2} cosine similarity
blended with current MW scores (score $= 0.6 \times \text{emb\_sim} + 0.4
\times \mathrm{MW}$). Outcomes are binary keyword-match: a success requires the agent's
retrieved memory set to contain task-relevant keywords, which is not
trivially satisfied---queries are paraphrased relative to memory
text, so matches depend on retrieval quality rather than surface
overlap. MW is updated with uniform credit $w = 1/k$.
All 20 seeds run locally.

\begin{figure}[t]
\centering
\includegraphics[width=\columnwidth]{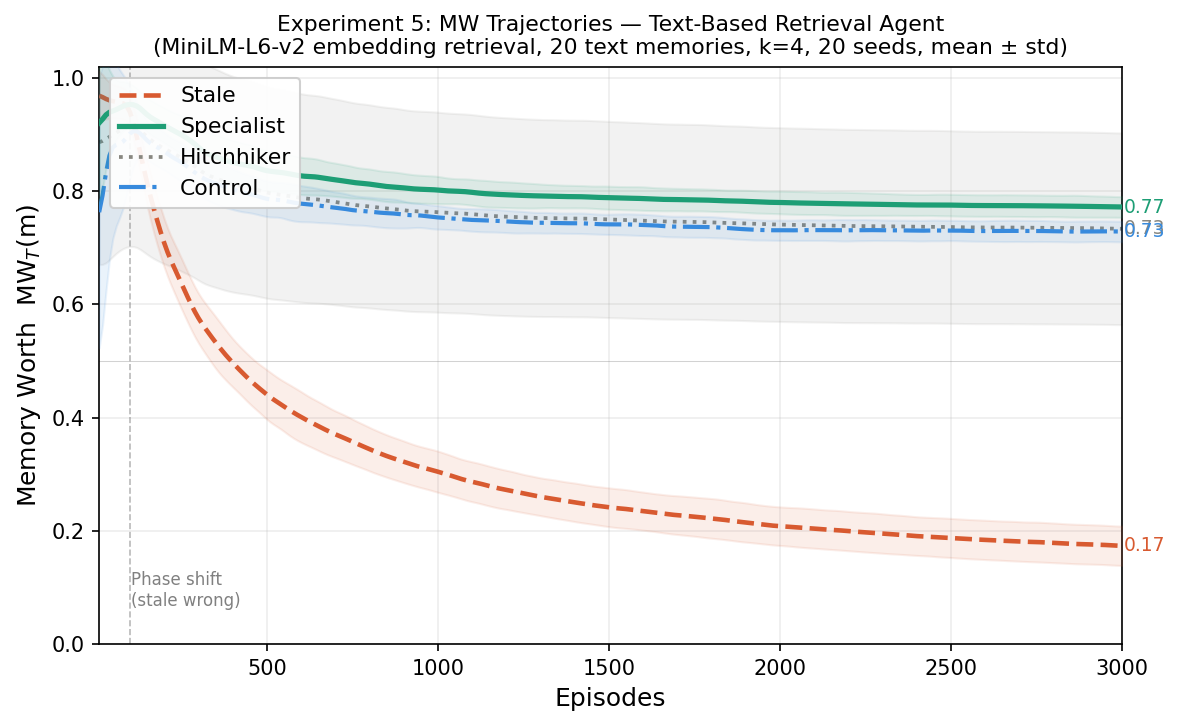}
\caption{Experiment~5: MW trajectories in a text-based retrieval agent
(\texttt{all-MiniLM-L6-v2} embedding retrieval, 3{,}000 episodes, 20 seeds,
mean $\pm$ std). The stale memory (dashed orange) peaks at $\approx 0.97$
in Phase~1, drops sharply after episode~100, crosses $\theta_L = 0.40$
near episode~300, and ends at $0.17$---well into the low-value category,
with no sign of convergence, warranting deprecation. The specialist (solid
green) stabilises at $0.77$. The hitchhiker (dotted grey, $0.77$) and
control (dash-dot blue, $0.73$) are both elevated and nearly
indistinguishable from the specialist---consistent with the co-retrieval
confound of Experiment~4: MiniLM routes Python queries to both specialist
and hitchhiker due to their semantic similarity.}
\label{fig:exp5}
\end{figure}

\begin{table}[t]
\caption{Experiment~5 MW summary (values from figure, 20 seeds). \texttt{all-MiniLM-L6-v2} retrieval, 3{,}000 episodes; stale crosses
$\theta_L=0.40$ near episode~300.}
\label{tab:exp5}
\vskip 0.1in
\begin{center}
\begin{small}
\begin{tabular}{lcccc}
\toprule
Memory type & MW @ ep100 & MW @ ep500 & MW @ ep3000 \\
\midrule
Stale       & ${\approx}0.97$ & ${\approx}0.40$ & $0.17$ \\
Specialist  & ${\approx}0.93$ & ${\approx}0.81$ & $0.77$ \\
Hitchhiker  & ${\approx}0.91$ & ${\approx}0.79$ & $0.77$ \\
Control     & ${\approx}0.78$ & ${\approx}0.77$ & $0.73$ \\
\bottomrule
\end{tabular}
\end{small}
\end{center}
\vskip -0.1in
\end{table}

\paragraph{Results.}
Figure~\ref{fig:exp5} and Table~\ref{tab:exp5} show the MW trajectories
across 3{,}000 episodes with \texttt{all-MiniLM-L6-v2} embedding retrieval.
The stale memory peaks at $\approx 0.97$ in Phase~1, drops sharply from
episode~100, crosses $\theta_L = 0.40$ near episode~300, and ends at
$0.17$ by episode~3{,}000---well into the low-value category and still
declining, making deprecation the warranted response. The specialist
stabilises at $0.77$. The hitchhiker ($0.77$) and control ($0.73$) are
both elevated and nearly indistinguishable from the specialist.
Experiment~5 reproduces the hitchhiker pathology from Experiment~4
under modern embedding retrieval, demonstrating that semantic retrieval
systems naturally induce the same co-retrieval confound predicted by the
synthetic analysis: MiniLM routes Python queries to both specialist (list
reversal) and hitchhiker (general list methods) due to their embedding
proximity, causing both to accumulate positive counts regardless of
causal contribution. The separation between stale ($0.17$) and all
other memory types ($0.73$--$0.77$) is at least 0.56 points and grows
monotonically after the phase shift, confirming that MW identifies the
genuinely degraded memory under modern semantic retrieval.

\paragraph{Limitations of this experiment.}
The early variance on specialist and hitchhiker is wide ($\pm 0.206$--$0.259$
at episode~100) because 20 text memories with $k=4$ retrieval gives each
memory substantial random-retrieval exposure in Phase~1. The task environment
is still scripted---outcomes come from keyword matching, not from a live
agent---so this is best characterised as a \emph{retrieval-realistic}
micro-experiment rather than a full deployment study. The simulation-only
limitation in Section~\ref{sec:limitations} applies.

\section{Toward Memory Governance Systems}
\label{sec:implications}

Memory Worth is designed as a foundation layer, not a complete system.
Adding two counters per memory unit to architectures that already log
retrievals and episode outcomes gives those systems a convergent
signal about which stored memories tend to co-occur with success.

\paragraph{Staleness detection.}
A memory whose $\mathrm{MW}_T(m)$ was high but has been declining
steadily is a candidate for re-verification. A low and stable
$\mathrm{MW}_T(m)$ with sufficient evidence $V_m \geq V_{\min}$ is a
reasonable trigger for retrieval suppression or human review, with the
caveat that the signal reflects association, not causation.

\paragraph{Retrieval prioritization.}
$\mathrm{MW}_T(m)$ can be added as a term in existing retrieval scoring
functions (e.g., alongside recency and embedding similarity) with no
other changes. This requires that the system logs which memories were
retrieved and what episode-level outcome followed---instrumentation that
many agent frameworks already support.

\paragraph{Open directions.}
Four extensions are motivated directly by the experiments.

First, \emph{contextual Memory Worth} $\mathrm{MW}(m \mid c)$
conditioned on a context variable $c$, motivated by Experiment~2:
global MW misevaluates specialist memories by $\sim$0.33$\rho$ under
task-difficulty confounding, and task-conditioned MW is the natural fix.
Experiment~2 treats task type as a known oracle label, but in deployment
the conditioning variable must be discovered or learned. Practical
candidates include: learned query clusters from embedding similarity
(unsupervised, no labels required); latent task embeddings derived from
the agent's tool-call sequences or reasoning traces; tool-path signatures
that identify structurally similar episodes; or user intent clusters from
session history. Each of these can be computed without ground-truth task
labels and provides a principled partition under which per-cluster MW
estimates are more informative than the global ratio.

Second, a \emph{Bayesian Beta-Bernoulli} formulation with
uncertainty-aware ranking (e.g.\ lower credible bound rather than
posterior mean) would give principled uncertainty quantification
beyond the point estimate, avoiding the need for a hard evidence
threshold $V_{\min}$. Experiment~1 shows the posterior mean
converges to the same ranking as MW at $K_T \gg \alpha+\beta$;
the remaining gain from the Bayesian formulation is in the low-count
regime and in credible-interval-based governance decisions. Third, a
\emph{non-stationary} variant using exponential moving averages would
extend the estimator to settings where $p^+(m)$ changes over time, at
the cost of the a.s.\ convergence guarantee. Fourth,
\emph{retrieval diversity} as a first-class system property: Experiment~4
shows that co-retrieval confounding persisted until $\approx$30\% of
episodes broke the habitual co-retrieval pair in that specific setting. Enforcing a retrieval diversity floor---analogous to the exploration floor in Experiment~3---should be treated as a governance invariant rather than an afterthought.

\section{Limitations}
\label{sec:limitations}

These limitations are not merely weaknesses of MW; they are design
constraints that any practical memory governance system built on MW
must explicitly address.

\paragraph{Association, not causation.}
$p^+(m)$ measures retrieval-outcome co-occurrence. A memory that rides
along with genuinely useful neighbors in the retrieval set will
accumulate positive counts even if it contributes nothing causally.
Conversely, a memory consistently retrieved alongside difficult tasks
will be penalized even if it is genuinely informative. Causal
identification of per-memory utility is an open problem that Memory
Worth does not resolve.

\paragraph{Assumption (A3) in realistic agents.}
The convergence theorem requires that retrieval and outcome are
conditionally independent given history. This holds in Experiment~1
by design (uniform random retrieval) but is violated in agents that
retrieve harder memories for harder tasks. Experiment~2 quantifies this
effect: global MW is negatively correlated with true utility
($\rho \approx -0.33$) under task-difficulty confounding. Conditioning
on the hard-task population moves the signal to $\rho \approx +0.14$,
showing partial recovery once task-type confounding is removed, but a
large residual gap to the unconfounded baseline ($0.89$) remains. Experiment~4 shows that co-retrieval
confounding is equally severe: memories always retrieved together are
indistinguishable even after 10{,}000 episodes; in that setting,
meaningful separation emerged only once $\approx$30\% of episodes broke
the co-retrieval pair, with the exact threshold depending on utility gap,
episode count, and retrieval topology. In practice, this means that
context identification and confound control are not optional refinements;
they are governance requirements for any deployment that uses MW to act
on memory quality.

\paragraph{Stationarity.}
The theorem assumes a stationary task distribution. Under distribution
shift, accumulated counts reflect a mixture of old and new utilities.
An exponential moving average can discount old observations but
loses the convergence guarantee.

\paragraph{Minimum evidence threshold.}
Estimates are unreliable for rarely retrieved memories. The
evidence-volume dimension $V_m$ in the taxonomy handles this in
principle (withhold classification until $V_m \geq V_{\min}$), but the
threshold is a hyperparameter. A Beta-Bernoulli formulation would give
automatic uncertainty quantification without a manual threshold.

\paragraph{Outcome signal quality.}
When episode outcomes are highly stochastic relative to the
contribution of any individual memory, the signal is diluted and
convergence is slow. Filtering high-variance episodes ($\epsilon$
signal) can mitigate this at the cost of reduced data efficiency.

\paragraph{Retrieval-realistic but not deployment-complete.}
Experiment~5 confirms both the direction and the magnitude of the MW signal
under neural embedding retrieval (\texttt{all-MiniLM-L6-v2}): the stale
memory crosses $\theta_L = 0.40$ by episode~300 and ends at $0.17$.
The task environment remains scripted---outcomes come from keyword matching
rather than a live agent---so this is best characterised as a
retrieval-realistic micro-study. Different embedding models may alter
the degree of hitchhiker confounding by changing neighbourhood topology,
but the mechanism itself is retrieval-model agnostic: any dense retrieval
system that routes semantically similar queries to multiple memories will
induce the same co-retrieval coupling. A live retrieval-augmented agent with natural outcome signals remains the highest-value remaining empirical validation step. More importantly, it would test whether the governance variables required by MW in practice
---context partitions, retrieval diversity, and uncertainty-aware action
thresholds---are recoverable from real agent logs.

\section{Conclusion}
\label{sec:conclusion}

This paper introduced Memory Worth, a per-memory online estimator of the
conditional success probability $p^+(m) = \Pr[y_t = +1 \mid m \in
\mathcal{M}_t]$, maintained as a ratio of two weighted retrieval counts.
We establish almost-sure convergence to $p^+(m)$ under a stationary
retrieval regime satisfying a minimum exploration condition, a conditional independence
assumption (A3), and an outcome stationarity condition (A6). The martingale argument
handles the dependence that arises when retrieval weights depend on past
history, but it does not handle the case where the retrieval decision
itself is coupled to the outcome---which is precisely when (A3) is
violated. This scope limitation is explicit throughout, and the three A3 violation experiments characterise the failure modes in concrete, quantified terms.

Empirically, Memory Worth reaches Spearman $\rho = 0.89 \pm 0.02$
in the controlled setting where all assumptions hold (Experiment~1).
Three A3 violation experiments reveal the failure modes quantitatively
(Experiments~2--4). A fifth retrieval-realistic micro-experiment
(\texttt{all-MiniLM-L6-v2} embedding retrieval, 3{,}000 episodes)
  shows the stale memory crossing $\theta_L = 0.40$ by episode~300
  and ending at $0.17$, while specialist MW stabilises at $0.77$---
  confirming the signal under modern semantic retrieval (Experiment~5).

The core observation is that every agent already runs an implicit
experiment: memories are retrieved, actions are taken, outcomes are
observed. Memory Worth formalizes what it means to actually read the
results of that experiment---and to act on them. The failure modes characterised here define the systems requirements for the next generation of memory governance architectures: task-conditioned estimation, retrieval diversity as a design constraint,
and uncertainty-aware ranking at low evidence counts. Memory Worth is not the full memory governance system, but it is
the minimal operational primitive such systems require.


\end{document}